\documentclass[10pt,journal,cspaper,compsoc,twocolumn]{IEEEtran}
\usepackage{bbold}
\usepackage{amsmath}
\usepackage{mathrsfs}
\usepackage[ruled,vlined,linesnumbered]{algorithm2e}
\providecommand{\SetAlgoLined}{\SetLine}

\usepackage{algpseudocode}
\usepackage[final]{graphicx}
\usepackage{epstopdf}
\usepackage{subfigure}
\usepackage{amsthm}
\usepackage{amsfonts}
\usepackage{bm}

\ifCLASSOPTIONcompsoc
\else
\fi

\ifCLASSINFOpdf
\else
\fi
\begin{document}

\newtheorem{theorem}{Theorem}
\renewcommand{\algorithmicrequire}{\textbf{Input:}}
\renewcommand{\algorithmicensure}{\textbf{Output:}}
\newcommand{\eg}{\emph{e.g.}\xspace}
\newcommand{\st}{\emph{s.t.}\xspace}
\newcommand{\ie}{\emph{i.e.}\xspace}
\newcommand{\etc}{\emph{etc.}\xspace}
\newcommand{\wrt}{\emph{w.r.t.}\xspace}
\newcommand{\etal}{\emph{et al.}\xspace}
\newcounter{save}\setcounter{save}{\value{section}}
{\def\addtocontents#1#2{}%
\def\addcontentsline#1#2#3{}%
\def\markboth#1#2{}%
%
\title{Semi-supervised Feature Analysis by Mining Correlations among Multiple Tasks}

\author{Xiaojun~Chang
        and~Yi~Yang 
\thanks{Xiaojun Chang and Yi Yang are with School of Information Technology and Electric Engineering, The University of Queensland, Australia.(e-mail: cxj273@gmail.com, yi.yang@uq.edu.au).}
}

\markboth{Journal of \LaTeX\ Class Files,~Vol.~11, No.~4, December~2012}%
{Shell \MakeLowercase{\textit{et al.}}: Bare Demo of IEEEtran.cls for Journals}

\maketitle

\begin{abstract}
In this paper, we propose a novel semi-supervised feature selection framework by mining correlations among multiple tasks and apply it to different multimedia applications. Instead of independently computing the importance of features for each task, our algorithm leverages shared knowledge from multiple related tasks, thus, improving the performance of feature selection. Note that we build our algorithm on assumption that different tasks share common structures. The proposed algorithm selects features in a batch mode, by which the correlations between different features are taken into consideration. Besides, considering the fact that labeling a large amount of training data in real world is both time-consuming and tedious, we adopt manifold learning which exploits both labeled and unlabeled training data for feature space analysis. Since the objective function is non-smooth and difficult to solve, we propose an iterative algorithm with fast convergence. Extensive experiments on different applications demonstrate that our algorithm outperforms other state-of-the-art feature selection algorithms.
\end{abstract}


\IEEEpeerreviewmaketitle

\section{Introduction}

\IEEEPARstart{I}{n} many computer vision and pattern recognition applications, dimension of data representation is normally very high. Recent studies have claimed that not all features in the high-dimensional feature space are discriminative and informative, since many features are often noisy or correlated to each other, which will deteriorate the performances of subsequent data analysing tasks \cite{bib_yu,bib_luis,bib_Shuang}. Consequently, feature selection is utilized to select a subset of features from the original high dimensional feature space \cite{bib_webimage,bib_ma,bib_Yang_Yi,bib_Huan,DBLP:journals/tsmc/YangO12}. It has twofold functions in enhancing performances of learning tasks. First, feature selection eliminates noisy and redundant information to get a better representation, thus facilitating classification and clustering tasks. Second, dimension of selected feature space becomes much lower, which makes the subsequent computation more efficient. Inspired by the motivations, much progress has been made to feature selection during last few years.

According to availability of class labels of training data, feature selection algorithms fall into two groups, \ie supervised feature selection and unsupervised feature selection. Supervised feature selection algorithms, for example, Fisher Score \cite{bib_Duda}, only use labeled training data for feature selection. With sufficient labeled training samples, supervised feature selection is reliable to train appropriate feature selection functions because of utilization of class labels. However, labeling a large amount of training samples manually is unrealistic in real-world applications. Recent works on semi-supervised learning have indicated that it is beneficial to leverage both labeled and unlabeled training data for data analysis. Motivated by the progress of semi-supervised learning, much research attention has been paid to semi-supervised feature selection. For example, Zhao \etal propose a semi-supervised feature selection algorithm based on spectral analysis. A common limitation of the existing supervised and semi-supervised feature selection algorithms is that they evaluate the importance of each feature individually, ignoring correlations between different features. To address this problem, some state-of-the-art algorithms are proposed to take feature correlations into consideration for feature selection. For example, \cite{bib_Nie} and \cite{bib_Shuang} implement their methods in a supervised way and Ma \etal design their approach in a semi-supervise way in \cite{bib_ma}. 

Another limitation of current feature selection algorithms is that they select features for each task individually, which fails to mine correlations among multiple related tasks. Recent researches have indicated that it is beneficial to learn multiple related tasks jointly \cite{bib_Rich,bib_convexMultitask,multifs,multi_task_sen}. Motivated by this fact, multi-task learning has been introduced to the field of multimedia. For instance, Yang \etal present a novel feature selection algorithm which leverages shared information from related tasks in \cite{bib_Yang_Yi}. Nevertheless, they design their algorithm in a supervised way.

The semi-supervised algorithm proposed in this paper combines the strengths of semi-supervised feature selection and multi-task learning. Both labeled and unlabeled training data are utilized for feature selection. Meanwhile, correlations between different features are taken into consideration to improve the performance of feature selection. 

\begin{figure*}[!ht]
\label{framework}
\centering
\includegraphics[scale=0.6]{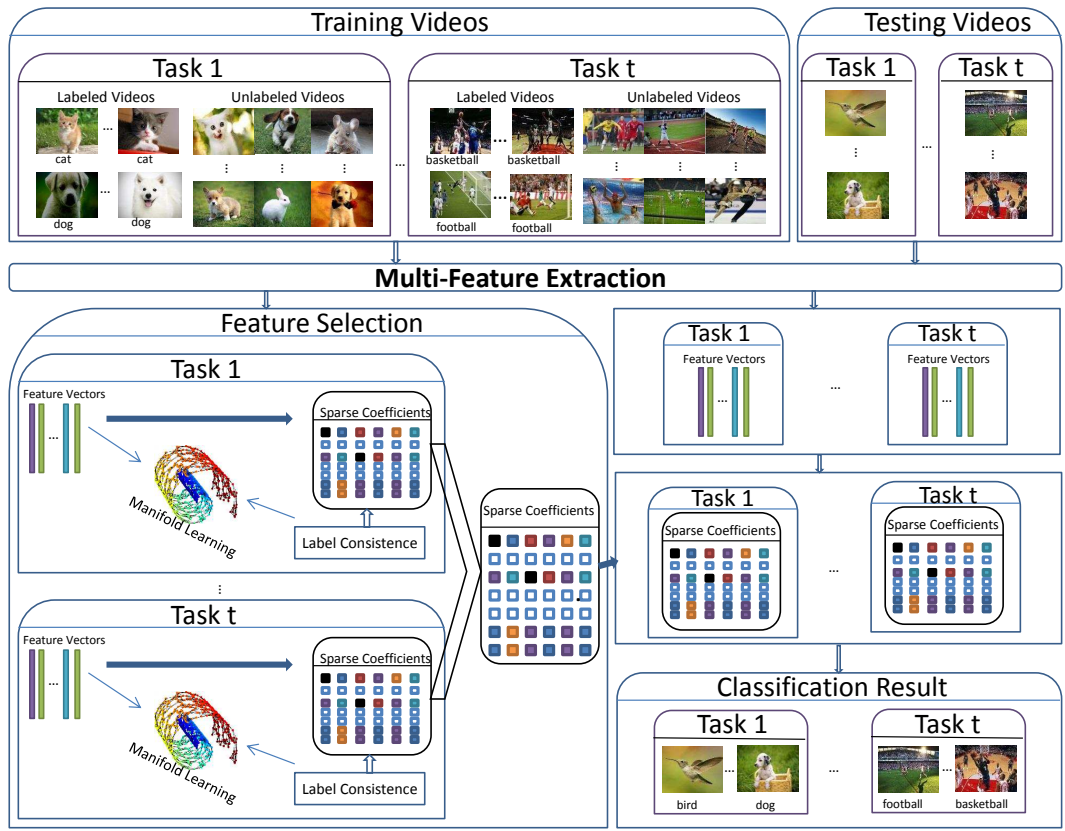}
\caption{The Illustration of general process of applying the proposed approach for video classification.}
\end{figure*}

We illustrate how the proposed algorithm works for video classification in Figure \ref{framework}. First, we represent all the training and testing videos as feature vectors. Then, sparse coefficients are learnt by exploiting relationships among different features and levearging knowledge from multiple related tasks. After selecting the most representative features, we can apply the sparse coefficients to the feature vectors of the testing videos for classification. 

We name our proposed algorithm Semi-supervised Feature selection by Mining Correlations among multiple tasks (SFMC). The main contributions of our work can be summarized as follows:

\begin{enumerate}
 \item We combine semi-supervised feature selection and multi-task learning into a single framework, which can select the most representative features with an insufficient amount of labeled training data per task.
 
 \item To explore correlations among multimedia data, we leverage the benefit of manifold learning into our framework. 
 
 \item Since the objective function is non-smooth and difficult to solve, a fast iterative algorithm to obtain the optimal solution is proposed. Experimental results on convergence demonstrate that the proposed algorithm converges within very few iterations.

\end{enumerate}

The rest of this paper is organized as follows: Section 2 summarizes the overview of the related work. A novel Semi-supervised Feature Selection by Mining Correlations among multiple tasks is proposed in section 3. We present our experimental results in section 4. The conclusion of our work is discussed in section 5.

\section{Related work}
In this section, we briefly review the related research on feature selection, semi-supervised learning and multi-task learning.

\subsection{Feature selection}
Previous works have claimed that feature selection is capable of selecting the most representative features, thus facilitating subsequent data analysing tasks \cite{bib_ZhengZhao} \cite{bib_featureSelection} \cite{DBLP:journals/tnn/XiangNMPZ12}.

Existing feature selection algorithms are designed in various ways. Classical feature selection algorithms, such as Fisher Score \cite{bib_Duda}, evaluate the weights of all features, rank them accordingly and select the most discriminating features one by one \cite{bib_l21norm}. Although these classical feature selection algorithms gain good performances in different applications, they have three main limitations. First, they only use labeled training data to exploit the correlations between features and labels for feature selection. Labeling a large amount of training data consumes a lot of human labor in real-world applications. Second, the most representative features are selected one by one, thus ignoring the correlations among different features. Third, they select features for each task independently, which fails to leverage the knowledge shared by multiple related tasks.

To overcome the aforementioned limitations, researchers have proposed multiple feature selection algorithms. $l_{2,1}$-norm regularization has been widely used in feature selection algorithms for its capability of selecting features across all data points with joint sparsity. For example, Zhao \etal propose an algorithm which selects features jointly based on spectral regression with $l_{2,1}$-norm constraint in \cite{bib_spec}. Nie \etal adopt $l_{2,1}$-norm on both regularization term and loss function in \cite{bib_Nie}. Yang \etal propose to select features by leveraging shared knowledge from multiple related tasks in \cite{bib_Yang_Yi}. However, their algorithms are all designed in a supervised way.

\subsection{Semi-supervised learning}
Semi-supervised learning has shown its promising performance in different applications \cite{bib_Xiaojin,DBLP:journals/tnn/SoaresCY12,DBLP:journals/tsmc/Wang11a,bib_Ira,cvprsemi,DBLP:conf/aaai/ChangNYH14}. With semi-supervised learning, unlabeled training data can be exploited to learn data structure, which can save human labor cost for labeling a large amount of training data \cite{bib_yiyang_rank,semiYun,DBLP:conf/pakdd/ChangSWLL14,ICMLsemi}. Hence, semi-supervised learning is beneficial in terms of both the human laboring cost and data analysis performance.

Graph Laplacian based semi-supervised learning has gained increasing interest for its simplicity and efficiency \cite{active_diversity}. Nie \etal propose a manifold learning framework based on graph Laplacian and compared its performance with other state-of-the-art semi-supervised algorithms in \cite{bib_manifold}. Ma \etal propose a semi-supervised feature selection algorithm built upon manifold learning in \cite{bib_ma}. In \cite{bib_Yang}, Yang \etal propose a new semi-supervised algorithm based on a robust Laplacian matrix for relevance feedback. Their algorithm has demonstrated its prominent performance. Therefore, we propose to leverage it in our feature selection framework. These previous works, however, independently select features for each task, which fails to consider correlations among multiple related tasks.

\subsection{Multi-task learning}
Multi-task learning has been widely used in many applications with the appealing advantage that it learns multiple related tasks with a shared representation \cite{bib_Rich} \cite{bib_convexMultitask} \cite{SIGKDD}. Recent researches have indicated that learning multiple related tasks jointly always outperforms learning them independently. Inspired by the progress of multi-task learning, researchers have introduced it to the field of multimedia and demonstrated its promising performance on multimedia analysis. For example, Yang \etal propose a novel multi-task feature selection algorithm which improves feature selection performance by leveraging shared information among multiple related tasks \cite{bib_Yang_Yi}. In \cite{bib_Yang_Yi}, Ma \etal apply knowledge adaptation to multimedia event detection and compare its performance with several state-of-the-art algorithms. Despite of their good performances, these classical algorithms are all implemented only with labeled training data.

\section{Methodology}
In this section, we describe the approach of our proposed algorithm in detail.

\subsection{Problem Formulation}

Suppose we are going to select features for $t$ tasks. The $l$-th task contains $n_l$ training data with $m_l$ data labeled. We can formulate the regularized framework for feature selection as follows:
\begin{equation}\label{obj1}
\min_{W_l} \sum_{l=1}^t ( loss(W_l) + \alpha g(W_l)) + \gamma \Omega(W),
\end{equation}
where $W_l$ is feature selection matrix for the $l$-th task, $W = [W_1, \cdots , W_t]$, $loss(W_l)$ is the loss function which evaluates consistency between features and labels, $g(W_l)$ is a regularization function, $\Omega(W)$ is a regularization term which is used to encode the common components of different feature selection functions, $\alpha$ and $\gamma$ are regularization parameters.

To step further, we first give the definitions of Frobenius norm and trace norm. Given an arbitrary matrix $M \in \mathbb{R}^{a \times b}$ where $a$ and $b$ are arbitrary numbers, its Frobenius norm is defined as $\|M\|_F$. The definition of its $l_{2,1}$-norm is:

\begin{equation}
\|M\|_{2,1} = \sum_{i=1}^a \sqrt{\sum_{j=1}^b M_{ij}^2} ,
\end{equation}
and the definition of its trace norm is:

\begin{equation}
\|M\|_* = Tr(MM^T)^{\frac{1}{2}},
\end{equation}
where $Tr(\cdot)$ denotes the trace operator. In the literature, there are many approaches to define the loss function. Following the works in \cite{bib_ma} \cite{bib_Yang_Yi}, we adopt the least square loss function for its simplicity and efficiency. Recent works \cite{bib_Nie} \cite{bib_l21norm} claim that minimizing the regularization term $\|W_l\|_{2,1}$ makes $W_l$ sparse, which demonstrates that $W_l$ is especially suitable for feature selection. Motivated by the works in \cite{bib_Guillaume} \cite{bib_Yang_Yi}, we propose to leverage shared knowledge among multiple related tasks by minimizing the trace norm of $W$. The objective function is given by:

\begin{equation}
\min_{W_l} \sum_{l=1}^t (loss(W_l) + \alpha \|W_l\|_{2,1}) + \gamma \|W\|_{*}
\end{equation}

State-of-the-art feature selection algorithms are implemented through supervised learning and select features for each task independently.  In our work, we want to incorporate multi-task learning and semi-supervised learning into \eqref{obj1}. We propose to leverage semi-supervised learning by adopting the Laplacian proposed in \cite{bib_Yang}. We adopt this Laplacian because it exploits both manifold structure and local discriminant information of multimedia data, thus resulting in better performance.

To begin with, let us define $X_l = [x_l^1, \cdots , x_l^{n_l}]$ as the training data matrix of the $l$-th task where $m_l$ data are labeled and $n_l$ is the total number of the training data of the $l$-th task. $x_l^i \in \mathbb{R}^d$ is the $i$-th datum of the $l$-th task. $Y_l = [y_l^1, \cdots , y_l^{m_l} , y_l^{m_l+1} , \cdots , y_l^{n_l}]^T \in \{0, 1\}^{n_l \times c_l}$ is the label matrix and $c_l$ denotes class number of the $l$-th task. $y_l^i|_{i=1}^{n_l} \in \mathbb{R}^{c_l}$ is the label vector with $c_l$ classes. ${Y_l}_{i,j} = 1$ if $x_l^i$ is in the $j$-th class of the $l$-th task while ${Y_l}_{i,j} = 0$ otherwise. For unlabeled datum $x_l^i$, $y_l^i$ is set to a zero vector. For any $d$, we define $\mathbf{1}_d \in \mathbb{R}^d$ as a column vector with all the elements equal to 1, $H_d = I - \frac{1}{d} \mathbf{1}_d \mathbf{1}_d^T \in \mathbb{R}^{d \times d}$ as a matrix for centering the data by subtracting the mean of the data. Note that $H_d = H_d^T = H_dH_d$. For each data point $x_l^i$ of the $l$-th task, we construct a local clique $\mathcal{N}_{lk}$ containing $x_l^i$ and its $k-1$ nearest neighbors. Euclidean distance is used to determine whether two given data points are within $k$ nearest neighbors in the original feature space. $G_l^i = \{ i_l^0, i_l^1, \cdots , i_l^{k-1} \}$ is index set of samples in $\mathcal{N}_{lk}$. $S_{li}$ denotes selection matrix with its elements $(S_{li})_{pq} = 1$ if $p = G_l^i\{q\}$ and $(S_{li})_{pq} = 0$ otherwise.

Inspired by \cite{bib_Yang}, we construct the Laplacian matrix by exploiting both manifold structure and local discriminant information. Denoting $L_{li} = H_k(X_l^TX_l + \lambda I)^{-1}H_k $, we compute the Laplacian matrix L as follows:
\begin{equation}
\begin{aligned}
L_l &= \sum_{i=1}^{n_l} S_{li}L_{li}S_{li}^T \\ &= [S_{l1}, \cdots , S_{ln_l}]\begin{bmatrix}
L_{l1} &  & \\
  &  \cdots & \\
   &    &   L_{ln_l} \\
\end{bmatrix}
[S_{l1}, \cdots , S_{ln_l}]^T.
\end{aligned}
\end{equation}

Note that Manifold Regularization is able to explore the manifold structure possessed by multimedia data \cite{bib_manifold} \cite{bib_YiYangManifold} \cite{seman}. By applying Manifold Regularization to the loss function in \eqref{obj1}, we have

\begin{equation}\label{obj2}
\begin{aligned}
\arg \min_{W,b} \sum_{l=1}^t Tr(W^TX_lL_lX_l^TW) + \alpha (\|W_l\|_{2,1} \\ + \beta \|X_{lL}^TW_l + \mathbf{1}_{n_l}b_l^T - Y_{lL}\|_F^2)) + \gamma \|W\|_{*},
\end{aligned}
\end{equation}
where $Tr(\cdot)$ denotes trace operator, $X_{lL}$ and $Y_{lL}$ are labeled training data and corresponding ground truth labels of the $l$-th task.

To make all labels of training data contribute to the optimization of $W_l$, we introduce a predicted label matrix $F_l = [f_{l1}, \cdots , f_{l_{n_l}} ] \in \mathbb{R}^{n_l \times c_l}$ for the training data of the $l$-th task. $f_{li} \in \mathbb{R}^{c_l}$ is the predicted label vector of $x_{li}$. According to \cite{bib_Xiaojin} \cite{bib_ma}, $F_l$ can be obtained as follows:

\begin{equation}\label{semif}
\arg \min_{F_l} Tr(F_l^TL_lF_l) + Tr((F_l-Y_l)^TU_l(F_l-Y_l)),
\end{equation}
where $U_l$ is the selection diagonal matrix of the $l$-th task. The diagonal element ${U_l}_{ii} = \infty$ if $x_{li}$ is labeled and ${U_l}_{ii} = 1$ otherwise. In the experiments, $10^6$ is used to approximate $\infty$.

Following the work in \cite{bib_ma}, we incorporate \eqref{semif} into \eqref{obj2}. At the same time, all the training data and corresponding labels are taken into consideration. Therefore, the objective function finally arrives at:

\begin{equation}\label{finalobj}
\begin{aligned}
\min_{F_l, W_l, b_l} \sum_{l=1}^t (Tr[(F_l - Y_l)^TU_l(F_l-Y_l)] + Tr(F_l^TL_lF_l) \\
+ \alpha(\|W_l\|_{2,1} + \beta \|X_l^TW_l + \mathbf{1}_{n_l}b_l^T - F_l\|_F^2)) + \gamma \|W\|_*
\end{aligned}
\end{equation}

From \eqref{finalobj} we can see that the proposed algorithm is capable of evaluating the informativeness of all features jointly for each task with the $l_{2,1}$-norm and the information from different tasks can be transferred from one to another with the trace norm.

\begin{algorithm}
\caption{Optimization Algorithm for SFMC }
 \SetAlgoLined
 \KwData{Training data $X_l|_{l=1}^t \in \mathbb{R}^{d \times n_l} $ \\
 ~~~~~~~~Training data labels $Y_l|_{l=1}^t \in \mathbb{R}^{n \times c}$ \\
 ~~~~~~~~~Parameters $\gamma$, $\alpha$ and $\beta$}
 \KwResult{\\
 ~~~~~~~~~Feature Selection Matrix $W_l|_{l=1}^t \in \mathbb{R}^{d \times c_l}$ }
 $l = 1$ \;
 \While{$l~ \leq t$} {
 Initialise $W_l|_{l=1}^t \in \mathbb{R}^{d \times c_l}$ \;
 Compute the Laplacian matrix $L_l|_{l=1}^t$ \;
 Compute the Selection matrix $U_l|_{l=1}^t$ \;
 $H_{n_l}=I_{n_l} - \frac{\mathbf{1}}{n_l} \mathbf{1}_{n_l} \mathbf{1}_{n_l}^T$ \;
 $P_l = (\alpha \beta H_{n_l} + U_l + L_l)^{-1}$ \;
 $R_l=X_lH_{n_l}(I_{n_l} - \alpha \beta P_l)H_{n_l}X_l^T $ \;
 $T_l = X_lH_{n_l}P_lU_lY_l$ \;
 }
 Set~$r=0$ \;
 Set $W_0 = [W_1, \cdots , W_t]$ \;
 \Repeat{Convergence}{
 l = 1 \;
 Compute the diagonal matrix as: $\widetilde{D}^r = (1/2)(W_rW_r^T)^{-1/2} $ \;
 \While{$l \leq t$}{
 Compute the diagonal matrix $D_l^r$ according to Eq. \eqref{Dlvalue} \;
 Update $W_l^r$ by $W_l^r = (R_l + \frac{\alpha}{\beta} D_l^r + \frac{\gamma}{\alpha \beta} \widetilde{D}^r)^{-1}T_l$ \;
 Update $F_l^r$ by $F_l^r = (\alpha \beta H_{n_l} + U_l + L_l)^{-1} (\alpha \beta H_{n_l} X_l^TW_l + U_lY_l)$ \;
 Update $b_l^r$ by $b_l^r = \frac{\mathbf{1}}{n_l} (F_l - X_l^TW_l)^T \mathbf{1}_{n_l}$ \;
 $l = l + 1$ \;
 }
 $W_{r+1} = [W_1, \cdots , W_t]$ \;
 $r = r + 1$ \;
 }
Return the optimal $W_l|_{l=1}^t$ and $b_l|_{l=1}^t$. 
\end{algorithm}

\subsection{Optimization}
The proposed function involves the $l_{2,1}$-norm and trace norm, which are difficult to solve in a closed form. We propose to solve this problem in the following steps.

By setting the derivative of (\ref{finalobj}) $w.r.t$ $b_l$ to $0$, we get
\begin{equation}
\label{bl}
b_l = \frac{\mathbf{1}}{n_l} (F_l - X_l^TW_l)^T \mathbf{1}_{n_l}
\end{equation}

Substituting $b_l$ in (\ref{finalobj}) with (\ref{bl}), we obtain

\begin{equation}\nonumber
\small
\begin{aligned}
 & \min_{F_l, W_l, b_l} \sum_{l=1}^t (Tr[(F_l - Y_l)^TU_l(F_l - Y_l)] + Tr(F_l^TL_lF_l)
 + \\ & \alpha ( \|W_l\|_{2,1}  + \beta \|X_l^TW_l + \frac{1}{n_l} \mathbf{1}_{n_l} \mathbf{1}_{n_l}^T(F_l - X_l^TW_l)  - F_l \|_F^2)) \\ & + \gamma \|W\|_*
\end{aligned}
\end{equation}

\begin{equation}\label{opdevnob}
\small
\begin{aligned}
\Rightarrow & \min_{F_l, W_l} \sum_{l=1}^t (Tr[(F_l - Y_l)^TU_l(F_l - Y_l)] + Tr(F_l^TL_lF_l) \\
& + \alpha (\|W_l\|_{2,1} + \beta \|H_{n_l}X_l^TW_l - H_{n_l}F_l \|_F^2)) + \gamma \|W\|_*
\end{aligned}
\end{equation}
where $H_{n_l}=I_{n_l} - \frac{\mathbf{1}}{n_l} \mathbf{1}_{n_l} \mathbf{1}_{n_l}^T$ is a centering matrix. By setting the derivative of \eqref{opdevnob} $w.r.t$ $F_l$ to $0$, we have
\begin{equation} \nonumber
2U_lF_l - 2U_lY_l + 2L_lF_l + \alpha \beta ( 2H_{n_l}F_l - 2H_{n_l}X_l^TW_l) = 0
\end{equation}

Therefore, we have
\begin{equation}
F_l = (\alpha \beta H_{n_l} + U_l + L_l)^{-1} (\alpha \beta H_{n_l} X_l^TW_l + U_lY_l)
\end{equation}

Denoting $P_l = (\alpha \beta H_{n_l} + U_l + L_l)^{-1}$ and $Q_l = \alpha \beta H_{n_l}X_l^TW_l + U_lY_l$, we have
\begin{equation}\label{Fv}
F_l = P_lQ_l
\end{equation}

By substituting $F_l$ into \eqref{opdevnob} with \eqref{Fv}, we can rewrite the objective function as follows:

\begin{equation}
\begin{aligned}
& \min_{Q_l, W_l} \sum_{l=1}^t (Tr[(P_lQ_l - Y_l)^TU_l(P_lQ_l - Y_l)] \\ 
& + Tr(Q_l^TP_l^TL_lP_lQ_l) + \alpha ( \|W_l\|_{2,1}  \\
&+ \beta \|H_{n_l}X_l^TW_l - H_{n_l}P_lQ_l\|_F^2)) + \gamma \|W\|_*
\end{aligned}
\end{equation}

As $Tr(Q_l^TP_l^TU_lY_l)=Tr(Y_l^TU_l^TP_lQ_l)$ and $Tr(\alpha \beta W_l^TX_lH_lP_lQ_l)=Tr(\alpha \beta Q_l^TP_l^TH_lX_l^TW_l)$, the objective function can be rewritten as follows:
\begin{equation}
\begin{aligned}
\min_{W_l} \sum_{l=1}^t (\alpha \beta Tr(W_l^TX_lH_{n_l}(I_{n_l} - \alpha \beta P_l)H_{n_l}X_l^TW_l) \\ - 2\alpha \beta Tr(W_l^TX_lH_{n_l}P_lU_lY_l) + \alpha \|W_l\|_{2,1}) + \gamma \|W\|_*
\end{aligned}
\end{equation}

Denoting $R_l=X_lH_{n_l}(I_{n_l} - \alpha \beta P_l)H_{n_l}X_l^T $, $T_l = X_lH_{n_l}P_lU_lY_l$ and $W_l = [w_l^1, \cdots , w_l^d]$, the objection function becomes:

\begin{equation}
\begin{aligned}
&\min_{W_l} \sum_{l=1}^t (\alpha \beta Tr(W_l^TR_lW_l) - 2\alpha \beta Tr(W_l^TT_l) \\ &+ \alpha Tr(W_l^TD_lW_l)) + \gamma \|W^T\widetilde{D}W\|_*,
\end{aligned}
\end{equation}
where $\widetilde{D}=(1/2)(WW^T)^{-1/2}$ and $D_l$ is a diagonal matrix which is defined as:

\begin{equation}
D_l = \begin{bmatrix}
\frac{1}{2\|w_l^1\|_2} & & \\
    &  \ddots  & \\
    &     & \frac{1}{2\|w_l^d\|_2}
\end{bmatrix}.
\label{Dlvalue}
\end{equation}

By setting the derivative $w.r.t$ $W_l$ to $0$, we have
\begin{equation}
\label{eqwl}
W_l = (R_l + \frac{\alpha}{\beta}D_l + \frac{\gamma}{\alpha \beta} \widetilde{D})^{-1}T_l
\end{equation}

As shown in Algorithm 1, an iterative algorithm is proposed to optimize the objective function \eqref{finalobj} based on the above mathematical deduction.

\subsection{Convergence Analysis}

In this section, we prove that Algorithm 1 converges by the following theorem.

\begin{theorem}
The objective function value shown in \eqref{finalobj} monotonically decreases in each iteration until convergence by applying Algorithm 1.
\end{theorem}

\begin{proof}
Suppose after the $r$-th iteration, we have obtained $F_l^r$, $b_l^r$ and $W_l^r$. According the definition of $D_l$ and $\widetilde{D}$, the convergence of Algorithm 1 corresponds to the following inequality:

\begin{equation}
\small
\begin{aligned}
& \sum_{l=1}^t Tr[(F_l^{r+1}-Y_l)^TU_l(F_l^{r+1}-Y_l)] + Tr((F_l^{r+1})^TL_lF_l^{r+1}) \\
& + \alpha ( \sum_{j=1}^d \frac{\|(w_l^{r+1})^j\|_2^2}{\|(w_l^{r})^j\|_2} + \beta \|X_l^TW_l^{r+1} + \mathbf{1}_{n_l}{b_l^{r+1}}^T - F_l^{r+1} \|_F^2) \\
& + Tr((W^{r+1})^T\frac{\gamma}{2}(W^r(W^r)^T)^{- \frac{1}{2}}W^{r+1}) \\
& \leq \sum_{l=1}^t tr[(F_l^{r}-Y_l)^TU_l(F_l^{r}-Y_l)] + Tr((F_l^{r})^TL_lF_l^{r}) \\
& + \alpha ( \sum_{j=1}^d \frac{\|(w_l^{r})^j\|_2^2}{\|(w_l^{r})^j\|_2} + \beta \|X_l^TW_l^{r} + \mathbf{1}_{n_l}{b_l^r}^T - F_l^{r} \|_F^2) \\
& + Tr((W^{r})^T\frac{\gamma}{2}(W^r(W^r)^T)^{- \frac{1}{2}}W^{r}) 
\end{aligned}
\end{equation}

Following the works in \cite{bib_Nie} \cite{bib_Yang_Yi} \cite{bib_ma}, we have:
\begin{equation}
\small
\label{eqleq}
\begin{aligned}
& \sum_{l=1}^t (Tr[(F_l^{r+1} - Y_l)^TU_l(F_l^{r+1} - Y_l)] + Tr((F_l^{r+1})^TL_lF_l^{r+1}) \\
& + \alpha (\sum_{j=1}^d \|(w_l^{r+1})^j \| + \beta \|X_l^TW_l^{r+1} + \mathbf{1}_{n_l}(b_l^{r+1})^T - F_l^{r+1} \|_F^2 )) \\
& + \frac{\gamma}{2} Tr(W^{r+1}(W^{r+1})^T(WW^T)^{- \frac{1}{2}}) \\
& \leq \sum_{l=1}^t (Tr[(F_l^{r} - Y_l)^TU_l(F_l^{r} - Y_l)] + Tr((F_l^{r})^TL_lF_l^{r}) \\
& + \alpha (\sum_{j=1}^d \|(w_l^{r})^j \| + \beta \|X_l^TW_l^{r} + \mathbf{1}_{n_l}(b_l^{r})^T - F_l^{r} \|_F^2 )) \\
& + \frac{\gamma}{2} Tr(W^{r}(W^{r})^T(W^{r}(W^r)^T)^{- \frac{1}{2}}).
\end{aligned}
\end{equation}

We can rewrite \eqref{eqleq} as follows:

\begin{equation}
\label{similar}
\small
\begin{aligned}
& \sum_{l=1}^t (Tr[(F_l^{r+1} - Y_l)^TU_l(F_l^{r+1} - Y_l)] + Tr((F_l^{r+1})^TL_lF_l^{r+1}) \\
& + \alpha (\sum_{j=1}^d \|(w_l^{r+1})^j \| + \beta \|X_l^TW_l^{r+1} + \mathbf{1}_{n_l}(b_l^{r+1})^T - F_l^{r+1} \|_F^2 )) \\
& + \frac{\gamma}{2} Tr((W^{r+1}(W^{r+1})^T)^{-\frac{1}{2}}) + \frac{\gamma}{2} Tr(W^{r+1}(W^{r+1})^T(WW^T)^{- \frac{1}{2}}) \\
& - \frac{\gamma}{2} Tr((W^{r+1}(W^{r+1})^T)^{-\frac{1}{2}}) \\
& \leq \sum_{l=1}^t (Tr[(F_l^{r} - Y_l)^TU_l(F_l^{r} - Y_l)] + Tr((F_l^{r})^TL_lF_l^{r}) \\
& + \alpha (\sum_{j=1}^d \|(w_l^{r})^j \| + \beta \|X_l^TW_l^{r} + \mathbf{1}_{n_l}(b_l^{r})^T - F_l^{r} \|_F^2 )) \\
& + \frac{\gamma}{2} Tr((W^{r}(W^{r})^T)^{-\frac{1}{2}}) + \frac{\gamma}{2} Tr(W^{r}(W^{r})^T(W^{r}(W^r)^T)^{- \frac{1}{2}}) \\
& - \frac{\gamma}{2} Tr((W^{r}(W^{r})^T)^{-\frac{1}{2}}).
\end{aligned}
\end{equation}

According to Lemma 1 in \cite{bib_Yang_Yi}, we have:

\begin{equation}
\small
\label{lemma}
\begin{aligned}
& \frac{\gamma}{2} Tr(W^{r+1}(W^{r+1})^T(W^{r}(W^{r})^T)^{- \frac{1}{2}}) - \gamma Tr((W^{r+1}(W^{r+1})^T \\
& \geq \frac{\gamma}{2} Tr(W^{r}(W^{r})^T(W^{r}(W^{r})^T)^{- \frac{1}{2}}) - \gamma Tr(W^{r}(W^{r})^T) 
\end{aligned}
\end{equation}

By deducting \eqref{lemma} from \eqref{similar}, we arrive at:

\begin{equation}
\small
\label{objdecrease}
\begin{aligned}
& \sum_{l=1}^t (Tr[(F_l^{r+1} - Y_l)^TU_l(F_l^{r+1} - Y_l)] + Tr((F_l^{r+1})^TL_lF_l^{r+1}) \\
& + \alpha (\|W_l^{r+1}\|_{2,1} + \beta \|X_l^TW_l^{r+1} + \mathbf{1}_{n_l}(b_l^{r+1})^T - F_l^{r+1} \|_F^2 )) \\
& + \gamma \|W^{r+1}\|_{*} \\
& \leq \sum_{l=1}^t (Tr[(F_l^{r} - Y_l)^TU_l(F_l^{r} - Y_l)] + Tr((F_l^{r})^TL_lF_l^{r}) \\
& + \alpha (\|W_l^r\|_{2,1} + \beta \|X_l^TW_l^{r} + \mathbf{1}_{n_l}(b_l^{r})^T - F_l^{r} \|_F^2 )) \\
& + \gamma \|W^r\|_{*})).
\end{aligned}
\end{equation}

Eq. \eqref{objdecrease} indicates that the objective function value decreases after each iteration. Thus, we have proved Theorem 1.
\end{proof}

Having Theorem 1, we can easily see that the algorithm converges.

\section{Experiments}
In this section, experiments are conducted to evaluate the performance of our algorithm on video classification, image annotation, human motion recognition and 3D motion data analysis, respectively. Additional experiments are conducted to study the performance \wrt influence of number of selected features and parameter sensitivity.

\subsection{Experiment Setup}

We use four different datasets in the experiment, including one video datasets CCV \cite{bib_ccv}, one image datasets NUSWIDE \cite{nus-wide-civr09}, one human motion dataset HMDB \cite{bib_HMDB} and one 3D motion skeleton dataset HumanEva \cite{bib_humaneva}. In order to demonstrate advantages of our algorithm, we compare its performance with the following approaches.

\begin{enumerate}
  \item \textbf{All Features:} We directly use the original features without feature selection as a baseline.
  
  \item \textbf{Fisher Score:} This is a classical feature selection method, which evaluates importances of features and selects the most discriminating features one by one \cite{bib_Duda}.
  
  \item \textbf{Feature Selection via Joint $l_{2,1}$-Norms Minimization (FSNM):} Joint $l_{2,1}$-norm minimization is utilized on both loss function and regularization for joint feature selection \cite{bib_Nie}. 
  
  \item \textbf{SPEC:} It uses spectral graph theory to conduct feature selection \cite{bib_spec}.
  
  \item \textbf{Feature Selection with Shared Information among multiple tasks (FSSI):} It simultaneously learns multiple feature selection functions of different tasks in a joint framework \cite{bib_Yang_Yi}. Hence, it is capable to utilize shared knowledge between multiple tasks to facilitate decision making.
  
  \item \textbf{Locality Sensitive Semi-supervised Feature Selection (LSDF):} This is a semi-supervised feature selection based on two graph constructions, \ie within-class graph and between-class graph \cite{bib_lsdf}.
  
  \item \textbf{Structural Feature Selection with Sparsity (SFSS):} It combines strengths of joint feature selection and semi-supervised learning into a single framework \cite{bib_ma}. Labeled and unlabeled training data are both utilized for feature selection. Meanwhile, correlations between different features are taken into consideration. 
  
\end{enumerate}

In the experiments, a training set for each dataset is randomly generated consisting of $n$ samples, among which $m\%$ samples are labeled. The detailed settings are shown in Table \ref{setting}. The remaining data are used as testing data. We independently repeat the experiment 5 times and report the average results.

\begin{table}[tb]
\caption{SETTINGS OF THE TRAINING SETS}
\centering
\begin{tabular}{|c||r|c|c|}
\hline
Dataset & Size(n) & Labeled Percentage ($m$) \\
\hline
CCV & $4,000$ & $1, 5, 10, 25, 50, 100$ \\
\hline
NUS-WIDE & $5,000$ & $1, 5, 10, 25, 50, 100$ \\
\hline
HMDB & $3,000$ & $2, 5, 10, 25, 50, 100$ \\
\hline
HumanEVA & $3,000$ & $1, 5, 10, 25, 50, 100$ \\
\hline
\end{tabular}
\label{setting}
\end{table}

We have to tune two types of parameters in the experiments. One is the parameter $k$ that specifies $k$ nearest neighbors used to compute graph Laplacian. Following \cite{bib_ma}, we fix it at $15$. The other parameter is the regularization parameters, $\alpha$, $\beta$ and $\gamma$, which are shown in the objective function \eqref{finalobj}. These parameters are tuned in the range of $\{ 10^{-6}, 10^{-4}, 10^{-2}, 10^0, 10^{+2}, 10^{+4}, 10^{+6} \}$ and the best results are reported. Linear SVM is used as classifier. Mean average precision (MAP) is used to evaluate the performance.

\subsection{Video Classification}
First, we compare the performances of different algorithms in terms of video classification task using Columbia Consumer Video dataset (CCV) \cite{bib_ccv}. It consists of $9,317$ web videos over 20 semantic categories, in which $4,659$ videos are used as training data and $4, 658$ videos are used as testing data. The semantic categories include events like "basketball" and "parade", scenes like "beach" and "playground", and objects like "cat" and "dog", based on which we generate three different classification tasks. Since the original videos of this dataset have not been available on the internet, we directly use the STIP features with $5,000$ dimensional BoWs representation provided by \cite{bib_ccv}. We set the number of selected features as $\{2500, 3000, \cdots , 4500, 5000\}$ for all the algorithms, and report the best results.


\begin{table*}[!ht]
\renewcommand{\arraystretch}{1.3}
\caption{PERFORMANCE COMPARISON OF Video Classification (MAP $\pm$ STD) \wrt 1\%, 5\% AND 10\% LABELED TRAINING DATA}
\centering
\subtable[Subject 1]{
\begin{tabular}{|l||c|c|c|c|c|c|}
\hline
   &  1\% labeled  & 5\% labeled   &  10\% labeled & 25\% labeled &   50\% labeled & 100\% labeled \\
\hline
All Features  &  $0.083 \pm 0.019$      &  $0.238 \pm 0.023$    &    $0.295 \pm 0.028$ &  $0.352 \pm 0.025$  &  $0.423 \pm 0.021$  &  $0.476 \pm 0.024$  \\
\hline
FISHER &  $0.091 \pm 0.025$  &  $0.241 \pm 0.028$  &  $0.302 \pm 0.019$ & $0.357 \pm 0.024$ & $0.428 \pm 0.021$  &  $0.479 \pm 0.026$\\
\hline
SPEC & $0.086 \pm 0.019$ & $0.241 \pm 0.021$  &  $0.0307 \pm 0.023$  & $0.359 \pm 0.019$ & $0.431 \pm 0.021$  & $0.478 \pm 0.023$ \\
\hline
FSNM  &  $0.087 \pm 0.018$  &  $0.243 \pm 0.022$  &  $0.309 \pm 0.026$ & $0.361 \pm 0.024$  & $0.434 \pm 0.025$ & $0.480 \pm 0.021$ \\
\hline
FSSI &  $0.105 \pm 0.022$  &  $0.245 \pm 0.019$ &   $0.314 \pm 0.024$ & $0.373 \pm 0.026$ & $0.427 \pm 0.023$ & $0.482 \pm 0.024$\\
\hline
SFMC & $\mathbf{0.112 \pm 0.013}$  & $\mathbf{0.292 \pm 0.018}$  & $\mathbf{0.335 \pm 0.012}$ & $\mathbf{0.396 \pm 0.015}$ & $\mathbf{0.459 \pm 0.017}$ & $\mathbf{0.498 \pm 0.019}$  \\
\hline
\end{tabular}
}
\subtable[Subject 2]{

\begin{tabular}{|l||c|c|c|c|c|c|}
\hline
   &  1\% labeled  & 5\% labeled   &  10\% labeled  & 25\% labeled   & 50\% labeled  & 100\% labeled \\
\hline
All Features  &  $0.316 \pm 0.024$      &  $0.421 \pm 0.026$    &  $0.455 \pm 0.021$ & $0.486 \pm 0.23$  & $0.534 \pm 0.027$  & $0.562 \pm 0.025$      \\
\hline
FISHER &  $0.320 \pm 0.037$  &  $0.438 \pm 0.029$  & $0.478 \pm 0.019$ & $0.492 \pm 0.025$ &  $0.545 \pm 0.021$ &  $0.568 \pm 0.023$\\
\hline
SPEC & $0.322 \pm 0.023$ & $0.441 \pm 0.025$ & $0.482 \pm 0.023$ & $0.493 \pm 0.023$  &$0.548 \pm 0.019$  & $0.570 \pm 0.022$ \\
\hline
FSNM  &  $0.324 \pm 0.019$  &  $0.449 \pm 0.024$  &  $0.486 \pm 0.026$ & $0.498 \pm 0.028$ & $0.551 \pm 0.027$  &  $0.572 \pm 0.021$ \\
\hline
FSSI &  $0.336 \pm 0.028$ & $0.458 \pm 0.022$  &  $0.495 \pm 0.019$ & $0.516 \pm 0.025$ &  $0.562 \pm 0.023$  &  $0.578 \pm 0.027$ \\
\hline
SFMC & $\mathbf{0.387 \pm 0.021}$  & $\mathbf{0.524 \pm 0.020}$  & $\mathbf{0.535 \pm 0.012}$ & $\mathbf{0.564 \pm 0.025}$  & $\mathbf{0.594 \pm 0.027}$  & $\mathbf{0.602 \pm 0.023}$ \\
\hline
\end{tabular}
}
\subtable[Subject 3]{
\begin{tabular}{|l||c|c|c|c|c|c|}
\hline
   &  1\% labeled  & 5\% labeled   &  10\% labeled & 25\% labeled   & 50\% labeled  & 100\% labeled \\
\hline
All Features  &   $0.381 \pm 0.017$     &   $0.848 \pm 0.028$   &  $0.857 \pm 0.024$ & $0.867 \pm 0.025$  &  $0.895 \pm 0.021$  & $0.910 \pm 0.026$     \\
\hline
FISHER &  $0.392 \pm 0.021$  &  $0.855 \pm 0.019$  &  $0.862 \pm 0.023$ & $0.873 \pm 0.026$  &  $0.900 \pm 0.024$  &  $0.918 \pm 0.025$ \\
\hline
SPEC & $0.396 \pm 0.023$ & $0.858 \pm 0.024$  & $0.868 \pm 0.019$ & $0.878 \pm 0.021$ & $0.905 \pm 0.023$  &  $0.913 \pm 0.022$\\
\hline
FSNM  &  $0.398 \pm 0.018$  & $0.861 \pm 0.022$   & $0.871 \pm 0.021$  & $0.880 \pm 0.026$  & $0.910 \pm 0.023$  & $0.921 \pm 0.019$   \\
\hline
FSSI & $0.424 \pm 0.024$  & $0.864 \pm 0.018$  &  $0.873 \pm 0.018$  &  $0.884 \pm 0.022$  & $0.905 \pm 0.021$  &  $0.921 \pm 0.021$  \\
\hline
SFMC & $\mathbf{0.479 \pm 0.012}$  &  $\mathbf{0.874 \pm 0.010}$ & $\mathbf{0.886 \pm 0.016}$ & $\mathbf{0.904 \pm 0.19}$  &  $\mathbf{0.912 \pm 0.017}$  & $\mathbf{0.925 \pm 0.014}$ \\
\hline
\end{tabular}
}
\label{videoclasstable}
\end{table*}

We show the video classification results when different percentages of labeled training data are used in Table \ref{videoclasstable}. From the experimental results, we can get the following observations: 1) The performances of all the compared algorithms increase when we increase the number of labeled training data. 2) The proposed algorithm consistently gains the best performance. 3) With 5\% labeled training data, our algorithm significantly outperforms other algorithms. For example, for subject 2, our algorithm is better than the second best algorithm by 6.6\%. Yet the proposed algorithm gains smaller advantage with more labeled training data. 

\subsection{Image Annotation}

We use NUS-WIDE dataset \cite{nus-wide-civr09} to test the performance of our algorithm. This dataset includes 269648 images of 81 concepts. A 500 dimension Bag-of-Words feature based on SIFT descriptor is used in this experiment. We take each concept as a separate annotation task, thus resulting in 81 tasks. It is difficult to report all the results of these 81 tasks, so the average result is reported. In this experiment, we set the number of selected features as $\{250, 275, \cdots , 475 , 500\}$ and report the best results.

We illustrate the experimental results in Table \ref{TAB_NUSWIDE}. From the experimental results, we can observe that the proposed method gains better performance than the other compared algorithms. We give the detailed results with 1\%, 5\% and 10\% labeled training data. It can be seen that the proposed algorithm is more competitive with less labeled training data.


\begin{table*}[!ht]
\renewcommand{\arraystretch}{1.3}
\caption{PERFORMANCE COMPARISON OF Image Annotation (MAP $\pm$ STD) \wrt 1\%, 5\% AND 10\% LABELED TRAINING DATA}
\centering
\begin{tabular}{|l||c|c|c|c|c|c|}
\hline
   &  1\% labeled  & 5\% labeled   &  10\% labeled & 25\% labeled   & 50\% labeled  & 100\% labeled  \\
\hline
All Features  &  $0.045 \pm 0.009$      &  $0.066 \pm 0.007$    &    $0.089 \pm 0.008$    & $0.096 \pm 0.007$  &  $0.105 \pm 0.006$  &  $0.115 \pm 0.008$\\
\hline
FISHER &  $0.049 \pm 0.008$  &  $0.069 \pm 0.005$  &  $0.091 \pm 0.009$  & $0.102 \pm 0.007$  & $0.108 \pm 0.009$ &  $0.117 \pm 0.007$ \\
\hline
SPEC & $0.051 \pm 0.010$ & $0.071 \pm 0.012$ & $0.093 \pm 0.009$ & $0.103 \pm 0.011$ & $0.116 \pm 0.007$  & $0.119 \pm 0.008$ \\
\hline
FSNM  &  $0.052 \pm 0.011$  &  $0.073 \pm 0.008$  &  $0.095 \pm 0.006$  & $0.103 \pm 0.009$  & $0.112 \pm 0.010$  & $0.121 \pm 0.009$ \\
\hline
FSSI &  $0.058 \pm 0.005$  &  $0.079 \pm 0.009$ &   $0.104 \pm 0.007$  & $0.110 \pm 0.008$ &  $0.121 \pm 0.008$  & $0.129 \pm 0.011$ \\
\hline
SFMC & $\mathbf{0.066 \pm 0.003}$  & $\mathbf{0.091 \pm 0.004}$  & $\mathbf{0.108 \pm 0.002}$ & $\mathbf{0.115 \pm 0.006}$  & $\mathbf{0.123 \pm 0.008}$  &  $\mathbf{0.131 \pm 0.009}$ \\
\hline
\end{tabular}
\label{TAB_NUSWIDE}
\end{table*}

\begin{table*}[!ht]
\renewcommand{\arraystretch}{1.3}
\caption{PERFORMANCE COMPARISON OF Human Motion Recognition (MAP $\pm$ STD) \wrt 2\%, 5\% AND 10\% LABELED TRAINING DATA}
\centering
\subtable[Subject 1]{
\begin{tabular}{|l||c|c|c|c|c|c|}
\hline
   &  2\% labeled  & 5\% labeled   &  10\% labeled & 25\% labeled   & 50\% labeled  & 100\% labeled \\
\hline
All Features  &  $0.214 \pm 0.017$      &  $0.231 \pm 0.019$    &    $0.286 \pm 0.015$    &  $0.334 \pm 0.018$  &  $0.448 \pm 0.021$  & $0.486 \pm 0.023$\\
\hline
FISHER &  $0.285 \pm 0.021$  &  $0.326 \pm 0.023$  &  $0.359 \pm 0.022$  & $0.401 \pm 0.024$  & $0.466 \pm 0.019$  &  $0.494 \pm 0.025$  \\
\hline
SPEC & $0.292 \pm 0.023$ & $0.321 \pm 0.024$ & $0.346 \pm 0.021$ & $0.378 \pm 0.024$ & $0.471 \pm 0.019$ & $0.496 \pm 0.018$\\
\hline
FSNM  &  $0.298 \pm 0.019$  &  $0.316 \pm 0.021$  &  $0.339 \pm 0.019$ & $0.367 \pm 0.024$  & $0.463 \pm 0.023$  &  $0.494 \pm 0.025$ \\
\hline
FSSI &  $0.314 \pm 0.018$  &  $0.338 \pm 0.019$ &   $0.365 \pm 0.023$ & $0.399 \pm 0.021$  &  $0.489 \pm 0.024$  &  $0.509 \pm 0.022$ \\
\hline
SFMC & $\mathbf{0.349 \pm 0.015}$  & $\mathbf{0.362 \pm 0.019}$  & $\mathbf{0.389 \pm 0.018}$ & $\mathbf{0.423 \pm 0.021}$  & $\mathbf{0.512 \pm 0.020}$  & $\mathbf{0.518 \pm 0.017}$  \\
\hline
\end{tabular}
}
\subtable[Subject 2]{

\begin{tabular}{|l||c|c|c|c|c|c|}
\hline
   &  2\% labeled  & 5\% labeled   &  10\% labeled & 25\% labeled   & 50\% labeled  & 100\% labeled\\
\hline
All Features  &   $0.271 \pm 0.018$     &  $0.475 \pm 0.019$    &  $0.518 \pm 0.025$      & $0.539 \pm 0.023$  & $0.577 \pm 0.022$  &  $0.658 \pm 0.023$\\
\hline
FISHER & $0.274 \pm 0.023$   & $0.479 \pm 0.022$   &  $0.534 \pm 0.019$  &$0.562 \pm 0.023$  & $0.593 \pm 0.021$ & $0.667 \pm 0.018$\\
\hline
SPEC & $0.279 \pm 0.024$ & $0.481 \pm 0.021$ & $0.548 \pm 0.023$ & $0.569 \pm 0.019$ & $0.598 \pm 0.022$ & $0.672 \pm 0.026$ \\
\hline
FSNM  &  $0.283 \pm 0.021$  &  $0.482 \pm 0.019$  & $0.559 \pm 0.025$  & $0.575 \pm 0.024$  &  $0.602 \pm 0.023$  &  $0.679 \pm 0.024$  \\
\hline
FSSI &  $0.286 \pm 0.019$ & $0.501 \pm 0.023$  &  $0.569 \pm 0.019$ &  $0.586 \pm 0.021$  & $0.608 \pm 0.022$  &  $0.682 \pm 0.017$ \\
\hline
SFMC & $\mathbf{0.397 \pm 0.016}$  & $\mathbf{0.580 \pm 0.014}$  & $\mathbf{0.623 \pm 0.021}$ &  $\mathbf{0.641 \pm 0.019}$ & $\mathbf{0.652 \pm 0.021}$ &  $\mathbf{0.709 \pm 0.024}$\\
\hline
\end{tabular}
}
\subtable[Subject 3]{

\begin{tabular}{|l||c|c|c|c|c|c|}
\hline
   &  2\% labeled  & 5\% labeled   &  10\% labeled & 25\% labeled   & 50\% labeled  & 100\% labeled \\
\hline
All Features  &   $0.198 \pm 0.021$     &  $0.229 \pm 0.018$    &  $0.278 \pm 0.023$      &  $0.335 \pm 0.022$  &  $0.354 \pm 0.019$  &  $0.433 \pm 0.017$\\
\hline
FISHER & $0.214 \pm 0.017$   & $0.249 \pm 0.016$   &  $0.286 \pm 0.019$ & $0.340 \pm 0.021$  &  $0.359 \pm 0.023$  &  $0.433 \pm 0.019$ \\
\hline
SPEC & $0.221 \pm 0.019$ & $0.247 \pm 0.021$ & $0.291 \pm 0.025$ & $0.331 \pm 0.022$ & $0.357 \pm 0.021$ & $0.439 \pm 0.018$ \\
\hline
FSNM  &  $0.210 \pm 0.021$  &  $0.251 \pm 0.022$  & $0.294 \pm 0.019$  &  $0.337 \pm 0.024$  & $0.354 \pm 0.018$  &  $0.442 \pm 0.020$ \\
\hline
FSSI &  $0.232 \pm 0.019$ & $0.276 \pm 0.021$  &  $0.301 \pm 0.023$ &  $0.342 \pm 0.026$  &  $0.370 \pm 0.021$  &  $0.439 \pm 0.018$ \\
\hline
SFMC & $\mathbf{0.239 \pm 0.019}$  & $\mathbf{0.288 \pm 0.015}$  & $\mathbf{0.315 \pm 0.016}$  & $\mathbf{0.347 \pm 0.019}$  &  $\mathbf{0.372 \pm 0.021}$ &  $\mathbf{0.451 \pm 0.022}$ \\
\hline
\end{tabular}
}
\subtable[Subject 4]{

\begin{tabular}{|l||c|c|c|c|c|c|}
\hline
   &  2\% labeled  & 5\% labeled   &  10\% labeled  & 25\% labeled   & 50\% labeled  & 100\% labeled \\
\hline
All Features  &   $0.194 \pm 0.019$     &  $0.204 \pm 0.018$    &  $0.219 \pm 0.023$      & $0.246 \pm 0.021$  & $0.274 \pm 0.017$  &  $0.332 \pm 0.024$\\
\hline
FISHER & $0.210 \pm 0.023$   & $0.224 \pm 0.017$   &  $0.230 \pm 0.019$ & $0.247 \pm 0.021$  & $0.274 \pm 0.024$  &  $0.334 \pm 0.023$\\
\hline
SPEC & $0.204 \pm 0.017$ & $0.217 \pm 0.021$ & $0.225 \pm 0.024$ & $0.243 \pm 0.027$ & $0.271 \pm 0.019$  & $0.339 \pm 0.016$ \\
\hline
FSNM  &  $0.195 \pm 0.021$  &  $0.206 \pm 0.024$  & $0.212 \pm 0.023$ & $0.246 \pm 0.019$  &  $0.278 \pm 0.023$  &  $0.343 \pm 0.018$  \\
\hline
FSSI &  $0.216 \pm 0.017$ & $0.221 \pm 0.023$  &  $0.235 \pm 0.019$ &  $0.256 \pm 0.023$ & $0.284 \pm 0.025$  & $0.351 \pm 0.019$ \\
\hline
SFMC & $\mathbf{0.226 \pm 0.019}$  & $\mathbf{0.238 \pm 0.018}$  & $\mathbf{0.251 \pm 0.024}$ &  $\mathbf{0.264 \pm 0.021}$  &  $\mathbf{0.292 \pm 0.023}$ & $\mathbf{0.359 \pm 0.019}$ \\
\hline
\end{tabular}
}
\subtable[Subject 5]{

\begin{tabular}{|l||c|c|c|c|c|c|}
\hline
   &  2\% labeled  & 5\% labeled   &  10\% labeled & 25\% labeled   & 50\% labeled  & 100\% labeled \\
\hline
All Features  &   $0.256 \pm 0.019$     &  $0.305 \pm 0.021$    &  $0.342 \pm 0.025$      & $0.386 \pm 0.024$ &  $0.467 \pm 0.023$  & $0.503 \pm 0.021$\\
\hline
FISHER & $0.302 \pm 0.018$   & $0.360 \pm 0.023$   &  $0.375 \pm 0.021$ &  $0.394 \pm 0.024$  &  $0.475 \pm 0.023$ & $0.511 \pm 0.025$\\
\hline
SPEC & $0.274 \pm 0.023$ & $0.332 \pm 0.018$ & $0.365 \pm 0.021$ & $0.391 \pm 0.027$ & $0.478 \pm 0.026$ & $0.509 \pm 0.022$ \\
\hline
FSNM  &  $0.269 \pm 0.019$  &  $0.316 \pm 0.022$  & $0.356 \pm 0.019$ & $0.389 \pm 0.023$  &  $0.471 \pm 0.021$  & $0.506 \pm 0.024$  \\
\hline
FSSI &  $0.342 \pm 0.022$ & $0.377 \pm 0.023$  &  $0.397 \pm 0.019$  &  $0.413 \pm 0.021$  &  $0.512 \pm 0.025$  &  $0.528 \pm 0.019$\\
\hline
SFMC & $\mathbf{0.356 \pm 0.015}$  & $\mathbf{0.385 \pm 0.018}$  & $\mathbf{0.401 \pm 0.022}$ & $\mathbf{0.421 \pm 0.024}$  & $\mathbf{0.528 \pm 0.021}$  &  $\mathbf{0.541 \pm 0.019}$ \\
\hline
\end{tabular}
}
\label{TAB_HMR}
\end{table*}

\subsection{Human Motion Recognition}


We use HMDB video dataset \cite{bib_HMDB} to compare the algorithms in terms of human motion recognition. HMDB dataset consists of 6,766 videos which are associated with 51 distinct action categories. These categories can be categorized into five groups: 1) General facial actions, 2) Facial actions with object manipulation, 3) General body movements, 4) Body movements with object interaction, 5) Body movements for human interaction. Therefore, in this experiment, the five groups are considered as five different tasks. Heng \etal claim that motion boundary histograms (MBH) is an efficient way to suppress camera motion in \cite{bib_hengwang} and thus it is used to process the videos. A 2000 dimension Bag-of-Words feature is generated to represent the original data. We set the number of selected features as $\{1000, 1200, \cdots , 1800, 2000\}$ for all the algorithms and report the best results.

Table \ref{TAB_HMR} shows the experiment results of human motion recognition. From Table \ref{TAB_HMR}, we observe that our method outperforms other compared algorithms. This experiment can further provide evidence that our algorithm is more advantageous with insufficient number of labeled training data.

\subsection{3D Motion Data Analysis}
We evaluate the performance of our algorithm in terms of 3D motion data analysis using Human-Eva 3D motion database. There are five different types of actions in this database, including $boxing,~gesturing,~walking,~throw$-$catch~and~jogging$. Following the work in \cite{bib_Ning} \cite{bib_Yang_Yi2}, we randomly select $10,000$ samples of two subjects ($5,000$ per subject). We encode each action as a collection of 16 joint coordinates in 3D space and obtain a 48-dimensional feature vector. Joint Relative Features between different joints are computed on top of that, resulting a feature vector with 120 dimensions. We combine the two kinds of feature vectors and get a 168-dimensional feature. In this experiment, we consider the two subjects as two different tasks. The number of selected features are tuned from $\{100, 110, \cdots , 160\}$.

The experiment results are shown in Table \ref{TAB_HumanEVA}. Table \ref{TAB_HumanEVA} gives detailed results when 1\%, 5\% and 10\% training data are labeled. From the experiment results, we can observe that our algorithm consistently outperform the other compared algorithms and obtains more performance gain when small number of training data are labeled.


\begin{table*}[!ht]
\renewcommand{\arraystretch}{1.3}
\caption{PERFORMANCE COMPARISON OF 3D MOTION DATA ANALYSIS (MAP $\pm$ STD) \wrt 1\%, 5\% AND 10\% LABELED TRAINING DATA}
\centering
\subtable[Subject 1]{
\begin{tabular}{|l||c|c|c|c|c|c|}
\hline
   &  1\% labeled  & 5\% labeled   &  10\% labeled & 25\% labeled   & 50\% labeled  & 100\% labeled \\
\hline
All Features  &  $0.776 \pm 0.027$      &  $0.849 \pm 0.021$    &    $0.871 \pm 0.025$    &  $0.886 \pm 0.024$  &  $0.898 \pm 0.021 $  &  $0.900 \pm 0.026$\\
\hline
FISHER &  $0.777 \pm 0.019$  &  $0.861 \pm 0.029$  &  $0.880 \pm 0.019$  & $0.887 \pm 0.024$ & $0.901 \pm 0.025$  &  $0.905 \pm 0.023$ \\
\hline
SEPC & $0.775 \pm 0.025$ & $0.857 \pm 0.021$ & $0.881 \pm 0.019$ & $0.893 \pm 0.023$ & $0.903 \pm 0.018$ & $0.911 \pm 0.022$ \\
\hline
FSNM  &  $0.778 \pm 0.034$  &  $0.851 \pm 0.024$  &  $0.883 \pm 0.023$  &  $0.897 \pm 0.019$ & $0.910 \pm 0.026$  &  $0.918 \pm 0.023$\\
\hline
FSSI &  $0.780 \pm 0.028$  &  $0.889 \pm 0.024$ &   $0.894 \pm 0.025$ &  $0.904 \pm 0.023$ & $0.912 \pm 0.026$ & $0.921 \pm 0.021$ \\
\hline
SFMC & $\mathbf{0.785 \pm 0.018}$  & $\mathbf{0.892 \pm 0.021}$  & $\mathbf{0.908 \pm 0.012}$ & $\mathbf{0.912 \pm 0.021}$  & $\mathbf{0.917 \pm 0.018}$  &  $\mathbf{0.925 \pm 0.020}$ \\
\hline
\end{tabular}
}
\subtable[Subject 2]{

\begin{tabular}{|l||c|c|c|c|c|c|}
\hline
   &  1\% labeled  & 5\% labeled   &  10\% labeled & 25\% labeled   & 50\% labeled  & 100\% labeled \\
\hline
All Features  & $0.819 \pm 0.024$     &  $0.860 \pm 0.022$    &  $0.909 \pm 0.029$      &  $0.928 \pm 0.026$  &  $0.946 \pm 0.019$  &  $0.950 \pm 0.023$\\
\hline
FISHER & $0.835 \pm 0.021$   & $0.864 \pm 0.020$   &  $0.909 \pm 0.025$ & $0.926 \pm 0.018$  &  $0.946 \pm 0.023$  &  $0.951 \pm 0.021$ \\
\hline
SPEC & $0.831 \pm 0.023$ & $0.868 \pm 0.019$ & $0.913 \pm 0.026$ & $0.929 \pm 0.021$ & $0.957 \pm 0.024$ & $0.959 \pm 0.027$ \\
\hline
FSNM  &  $0.836 \pm 0.025$  &  $0.870 \pm 0.018$  & $0.921 \pm 0.023$  & $0.938 \pm 0.021$  & $0.964 \pm 0.024$  &  $0.965 \pm 0.022$ \\
\hline
FSSI &  $0.836 \pm 0.020$ & $0.884 \pm 0.026$  &  $0.922 \pm 0.024$ & $0.947 \pm 0.022$  & $0.961 \pm 0.023$  &  $0.962 \pm 0.019$ \\
\hline
SFMC & $\mathbf{0.847 \pm 0.023}$  & $\mathbf{0.894 \pm 0.019}$  & $\mathbf{0.948 \pm 0.026}$ & $\mathbf{0.954 \pm 0.023}$  & $\mathbf{0.973 \pm 0.025}$  &  $\mathbf{0.975 \pm 0.022}$ \\
\hline
\end{tabular}
}
\label{TAB_HumanEVA}
\end{table*}

\subsection{Comparison with Other Semi-Supervised Feature Selection Methods}

In this section, experiments are conducted on CCV to compare the proposed algorithm with two state-of-the-art semi-supervised feature selection algorithms. Following the above experiments, 1\%, 5\%, 10\%, 25\%, 50\% and 100\% training data are labeled in this experiment. We show the experiment results in Figure \ref{compareSemi}. We can observe that our method consistently outperforms both LSDF and SFSS. Visible advantages are gained when only few training data are labeled, such as 1\% or 5\% labeled training data. From this result, we can conclude that it is beneficial to leverage shared information from other related tasks when insufficient number of training data are labeled.
 
\begin{figure*}[!ht]
\centering
\subfigure[]{
\includegraphics[width=0.32\linewidth]{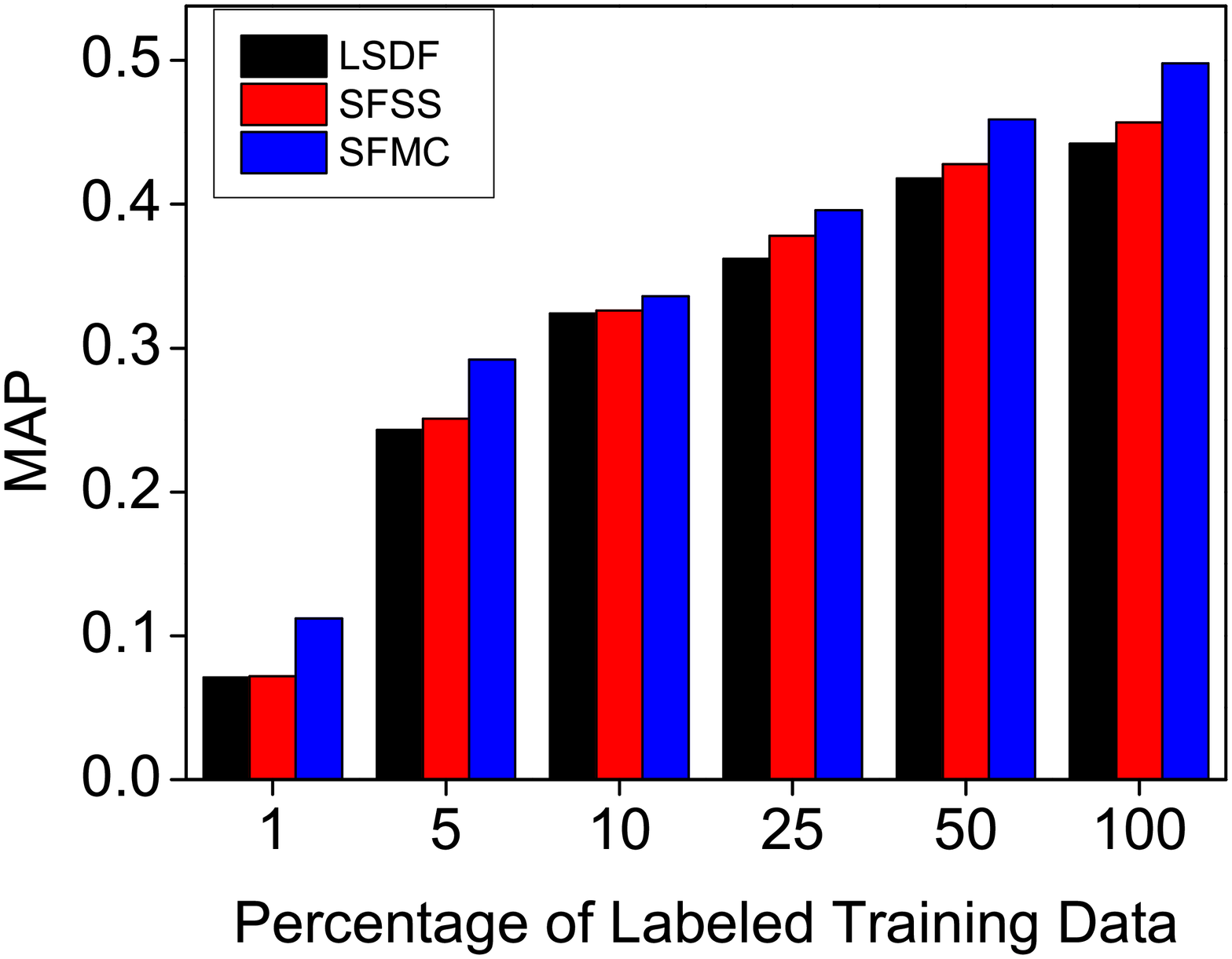}}
\subfigure[]{
\includegraphics[width=0.32\linewidth]{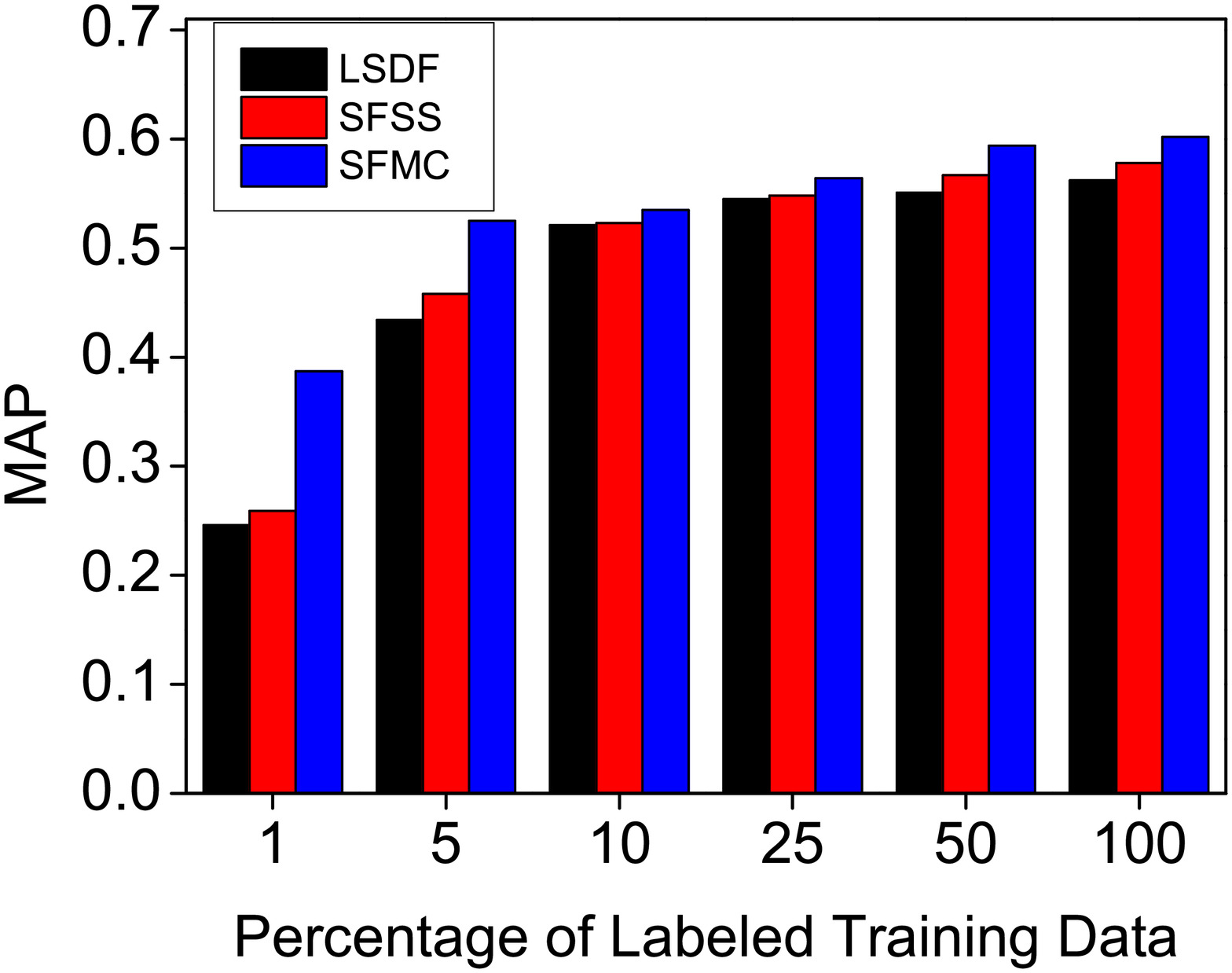}}
\subfigure[]{
\includegraphics[width=0.32\linewidth]{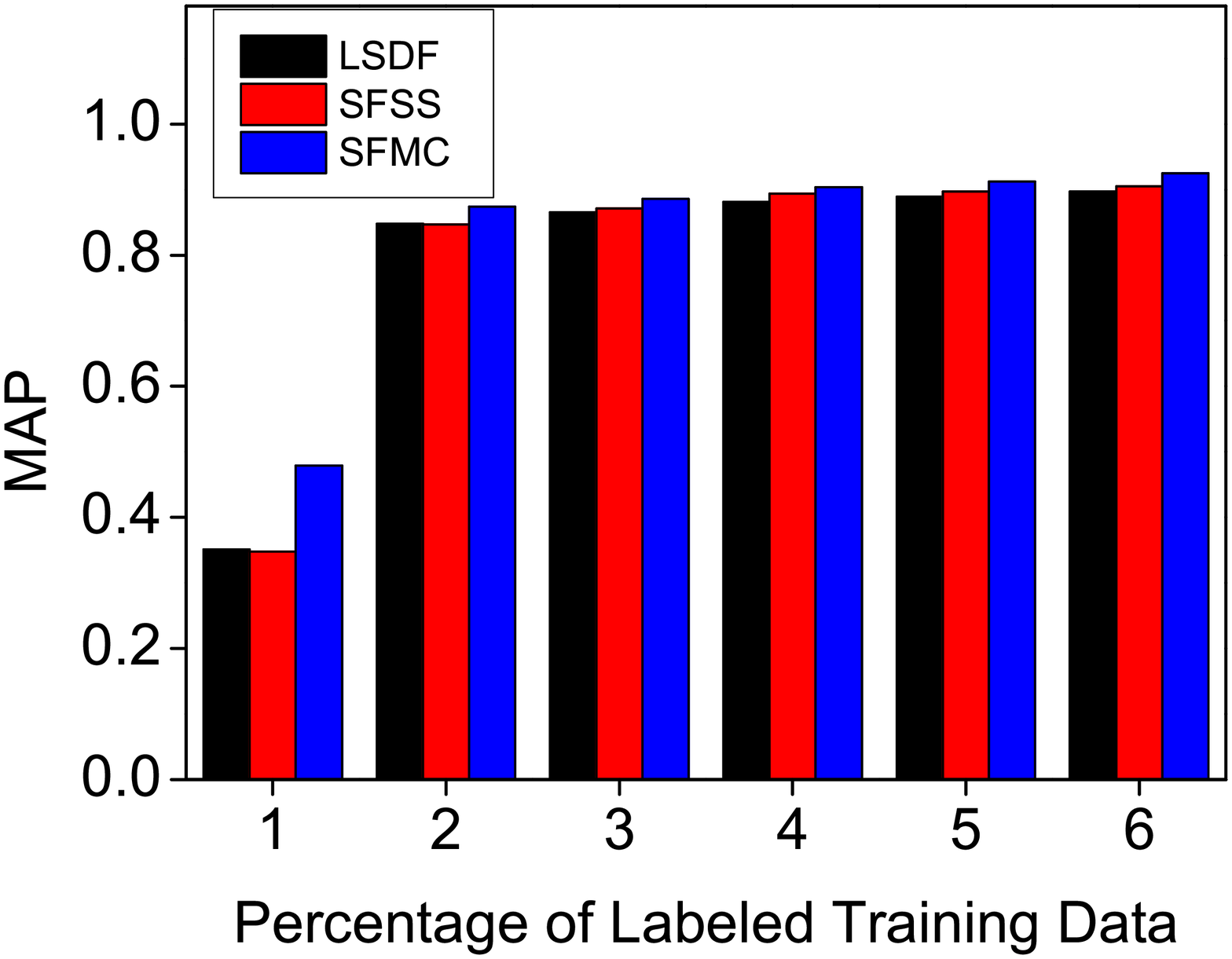}}
\caption{PERFORMANCE COMPARISON OF 3D MOTION DATA ANALYSIS \wrt PERCENTAGE OF LABELED TRAINING DATA. WE CAN OBSERVE THAT THE PROPOSED ALGORITHM YIELDS TOP PERFORMANCES ON BOTH THE TASKS. (a) SUBJECT 1. (b) SUBJECT 22.} 
\label{compareSemi}
\end{figure*}

\subsection{Parameter Sensitivity}
We study the influences of the four parameters $\alpha$, $\beta$, $\gamma$ and the number of selected features using CCV database with $1\%$ labeled training data. First, we fix $\gamma$ and the number of selected features at 1 and 3500 respectively, which are the median values of the tuned range of the parameters. The experimental results are shown in Figure \ref{parsen_alphabeta}. It can be seen that the performance of our algorithm varies when the parameters ($\alpha$ and $\beta$) change. More specifically, MAP is higher when $\alpha$ and $\beta$ are comparable. Then, $\alpha$ and $\beta$ are fixed. Figure \ref{parsen_gammanum} shows the parameter sensitivity results. Note that the shared information among multiple feature selection functions $\{W_1, \cdots , W_t\}$ by the parameter $\gamma$. From this figure, we can see that mining correlations between multiple related tasks is beneficial to improve the performance. We can also notice that better performances are gained when the number of features is around 3500 and 4000.  

\begin{figure*}[!ht]
\centering
\subfigure[]{
\includegraphics[width=0.32\linewidth]{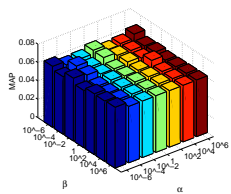}}
\subfigure[]{
\includegraphics[width=0.32\linewidth]{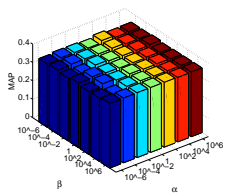}}
\subfigure[]{
\includegraphics[width=0.32\linewidth]{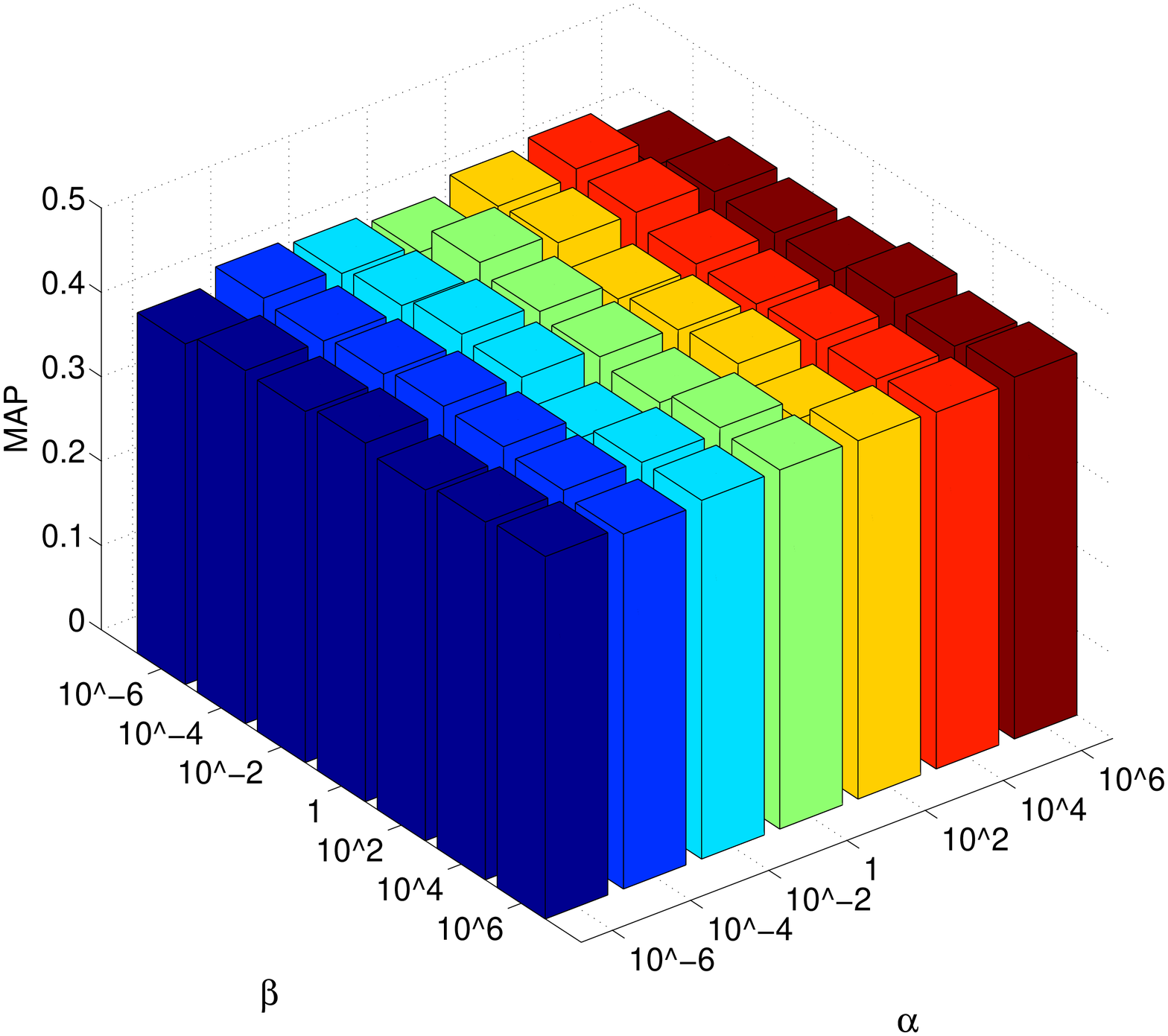}}
\caption{MAP with different $\alpha$ and $\beta$ while keeping $\gamma$ and feature numbers fixed on CCV database. (a) SUBJECT 1. (b) SUBJECT 2. (c) SUBJECT 3} 
\label{parsen_alphabeta}
\end{figure*}

\begin{figure*}[!ht]
\centering
\subfigure[]{
\includegraphics[width=0.32\linewidth]{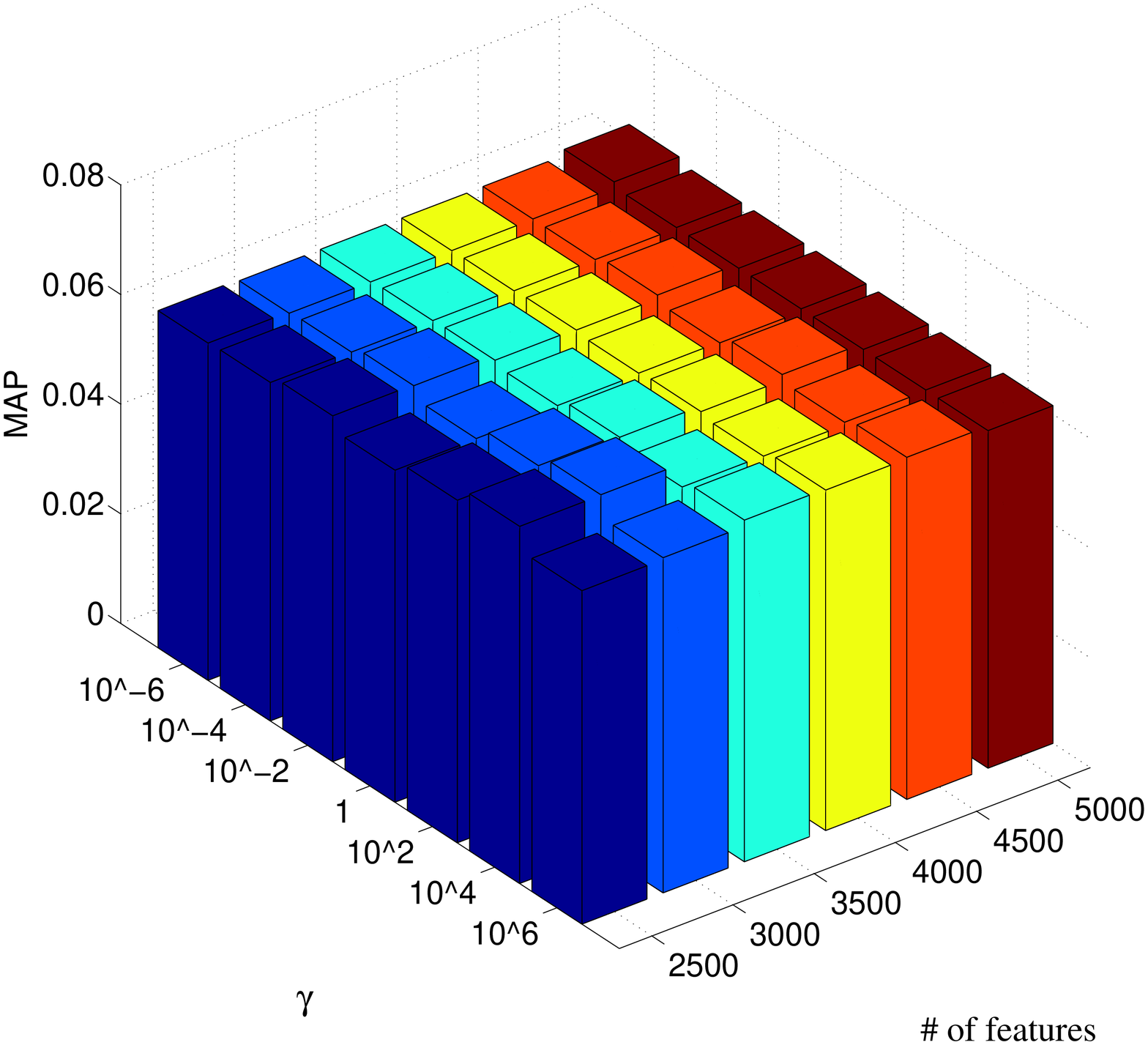}}
\subfigure[]{
\includegraphics[width=0.32\linewidth]{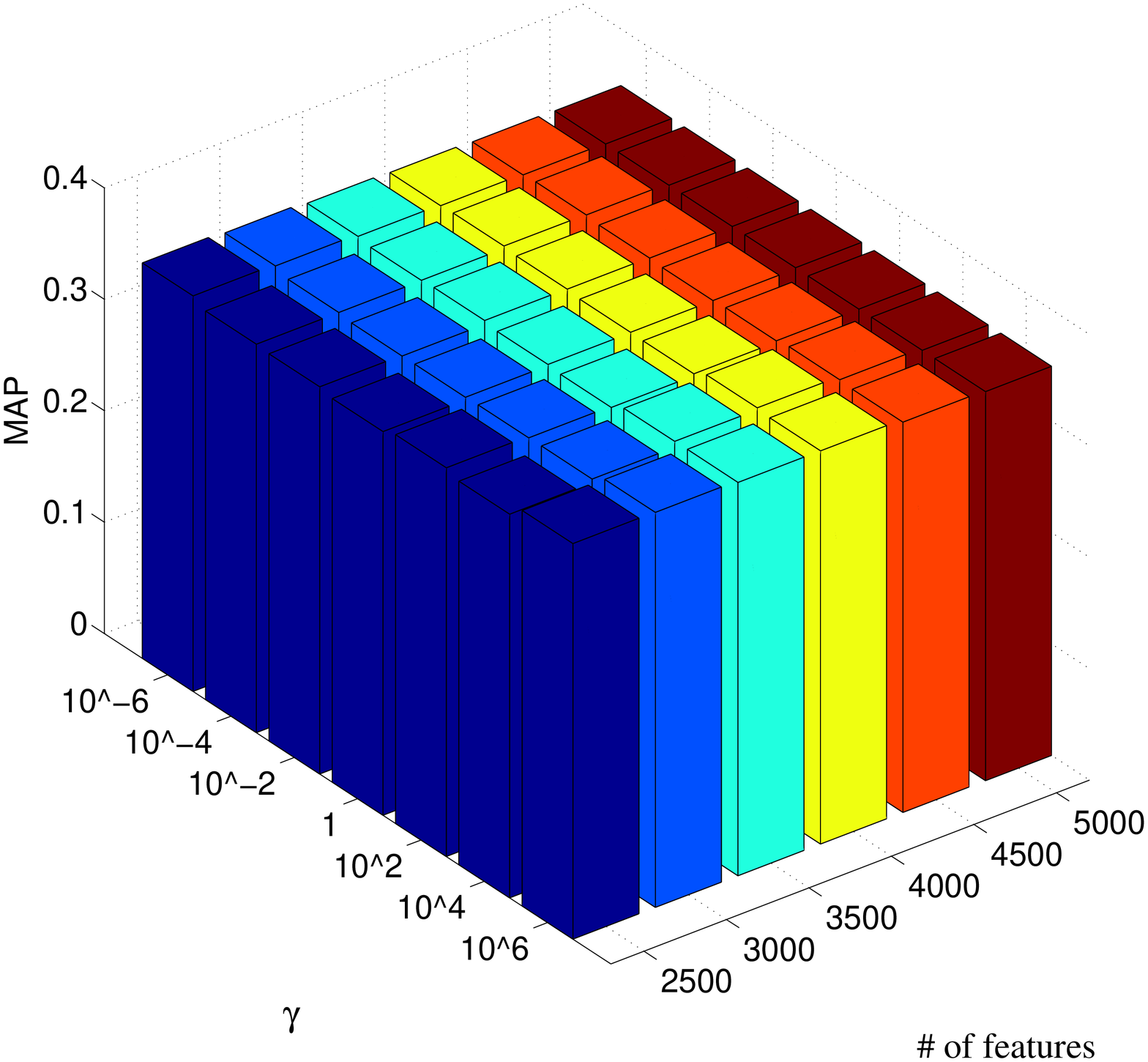}}
\subfigure[]{
\includegraphics[width=0.32\linewidth]{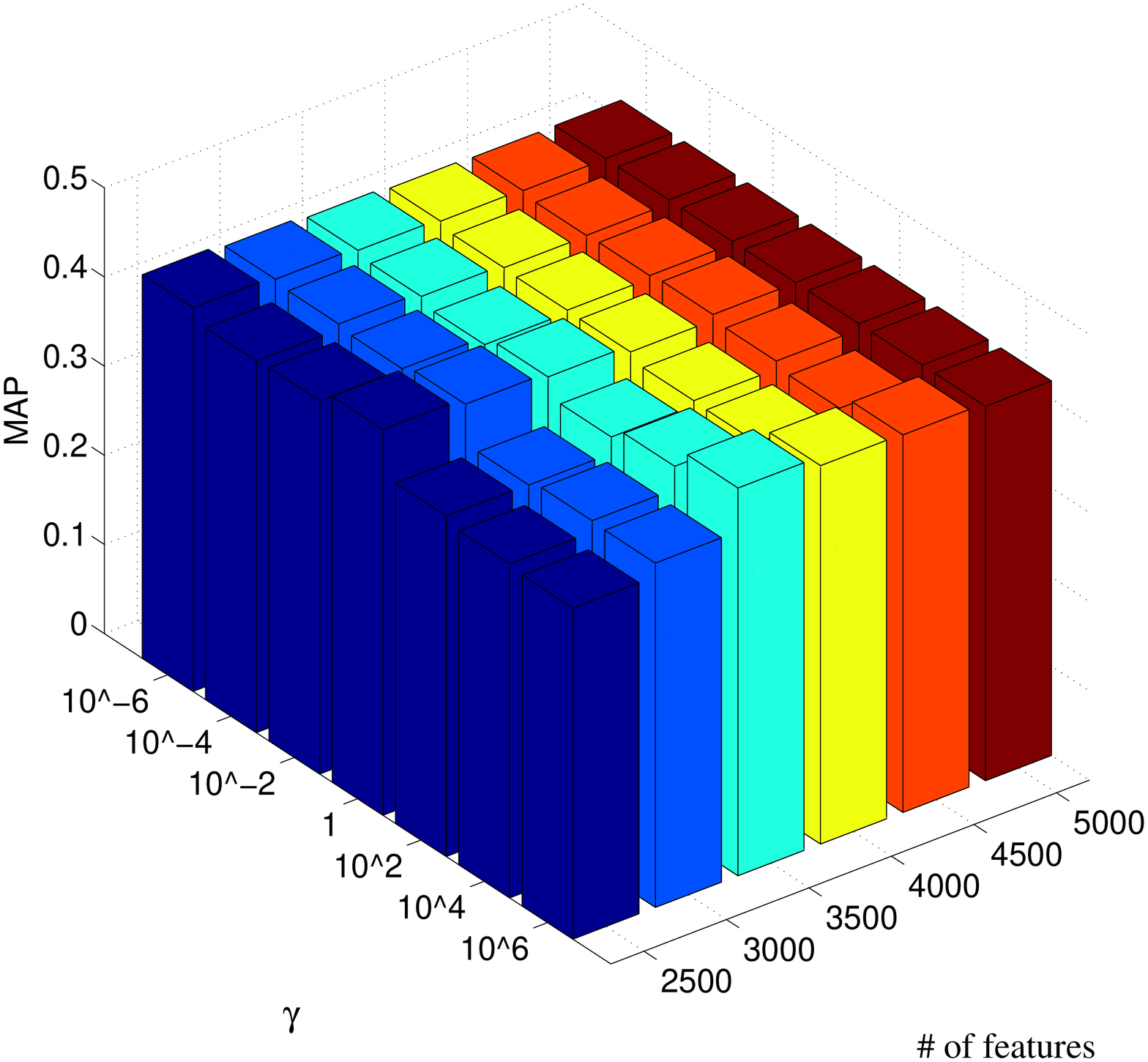}}
\caption{MAP with different $\gamma$ and the number of features while keeping $\alpha$ and $\beta$ fixed on CCV database. (a) SUBJECT 1. (b) SUBJECT 2. (c) SUBJECT 3} 
\label{parsen_gammanum}
\end{figure*}

%

\section{Conclusion}
In this paper, we have proposed a new semi-supervised feature analysis method. This method is able to mine correlations between different features and leverage shared information between multiple related tasks. Since the proposed objective function is non-smooth and difficult to solve, we propose an iterative and effective algorithm. To evaluate performances of the proposed method, we apply it to different applications, including video classification, image annotation, human motion recognition and 3D motion data analysis. The experimental results indicate that the proposed method outperforms the other compared algorithms for different applications.

{
\bibliographystyle{IEEEtran}
\bibliography{tkde}
}

\end{document}